\documentclass[12pt]{article}

\usepackage{tocloft}
\usepackage{graphicx}
\usepackage{natbib}
\usepackage[protrusion=true,expansion=true]{microtype}
\usepackage{amsmath,amsthm,amsfonts,amssymb,mathrsfs,dsfont,mathtools}
\usepackage{booktabs}
\usepackage{url}
\usepackage{float}
\usepackage{color}
\usepackage{parskip}
\usepackage[explicit]{titlesec}
\usepackage{threeparttable}
\usepackage{setspace}
\usepackage[margin=0.85in]{geometry}
\usepackage[font = onehalfspacing]{caption}
\usepackage[usenames,dvipsnames]{xcolor}
\usepackage[hyperfootnotes=true]{hyperref}
\usepackage[labelfont=bf]{caption}
\usepackage{authblk}
\usepackage{caption}
\usepackage[titletoc,toc,title]{appendix}
\usepackage{enumerate}
\usepackage{soul}
\usepackage{subfig}
\usepackage{accents}
\usepackage{subfiles} 
\usepackage[capitalise, nameinlink]{cleveref}
\usepackage{dsfont}
\usepackage{enumitem}

\usepackage{changepage}
\usepackage[noend]{algpseudocode}
\usepackage{tocloft}
\usepackage[protrusion=true,expansion=true]{microtype}
\usepackage{amsmath,amsthm,amsfonts,amssymb,mathrsfs,dsfont,mathtools}
\usepackage{float}
\usepackage{color}
\usepackage{parskip}
\usepackage[explicit]{titlesec}
\usepackage{threeparttable}
\usepackage{setspace}
\usepackage[font = onehalfspacing]{caption}
\usepackage[dvipsnames]{xcolor}
\usepackage[hyperfootnotes=true]{hyperref}
\usepackage[labelfont=bf]{caption}
\usepackage{authblk}
\usepackage{enumerate}
\usepackage{soul}
\usepackage{subfig}
\usepackage{accents}
\usepackage[capitalise, nameinlink]{cleveref}
\usepackage{diagbox}
\usepackage{array,multirow}

\usepackage[utf8]{inputenc} 
\usepackage[T1]{fontenc}    
\usepackage{hyperref}       
\usepackage{url}            
\usepackage{booktabs}       
\usepackage{amsfonts}       
\usepackage{nicefrac}       
\usepackage{microtype}      
\usepackage{lipsum}		
\usepackage{graphicx}
\usepackage{natbib}
\usepackage{doi}
\usepackage{subfiles}

\usepackage{adjustbox}
\usepackage{multirow}
\usepackage{amsmath}
\usepackage[section]{placeins}
\usepackage{algorithm}
\usepackage{algpseudocode}
\usepackage{dsfont}
\usepackage{comment}

\DeclareUnicodeCharacter{2212}{-}

\setstcolor{red} 

\numberwithin{equation}{section}

\hypersetup{colorlinks = true, citecolor = MidnightBlue, linkcolor = MidnightBlue, urlcolor = MidnightBlue}

\titleformat{\section}{\normalfont\large\bfseries}{\thesection}{1em}{#1}
\titleformat{\subsection}{\normalfont\normalsize\bfseries}{\thesubsection}{1em}{#1}
\titleformat{\subsubsection}{\normalfont\normalsize\itshape}{\thesubsubsection}{1em}{#1}
\titlespacing\section{0pt}{12pt plus 4pt minus 2pt}{6pt plus 2pt minus 2pt}
\titlespacing\subsection{0pt}{12pt plus 4pt minus 2pt}{3pt plus 2pt minus 3pt}
\titlespacing\subsubsection{0pt}{12pt plus 4pt minus 2pt}{0pt plus 2pt minus 3pt}

\def\boxit#1{\vbox{\hrule\hbox{\vrule\kern6pt
          \vbox{\kern6pt#1\kern6pt}\kern6pt\vrule}\hrule}}

\definecolor{orange}{rgb}{1,0.5,0}
\definecolor{MyDarkBlue}{rgb}{0,0.08,0.45}

\definecolor{MyDarkGreen}{RGB}{0,100,0}

\newtheorem{corollary}{Corollary}[section]

\newtheorem{proposition}{Proposition}[section]

\newtheorem{assumption}{Assumption}[section]

\def\boxit#1{\vbox{\hrule\hbox{\vrule\kern6pt
          \vbox{\kern6pt#1\kern6pt}\kern6pt\vrule}\hrule}}

\definecolor{orange}{rgb}{1,0.5,0}
\definecolor{MyDarkBlue}{rgb}{0,0.08,0.45}

\begin{document}
\title{\Large \bfseries Deep catastrophic-risk-aware reinforcement learning\thanks{We thank the Natural Sciences and Engineering Research Council of Canada (Godin: RGPIN-2024-04593) for their financial support.}
 } 
 \author[a]{Carlos Octavio Pérez Mendoza}
\author[a,b]{Fr\'ed\'eric Godin\thanks{Corresponding author. \vspace{0.2em} \newline
{\mbox{\hspace{0.47cm}} \it Email addresses:} 
\href{mailto:carlos.octavio.perez92@gmail.com}{carlos.octavio.perez92@gmail.com} (Carlos Pérez Octavio Mendoza),
\href{mailto:frederic.godin@concordia.ca}{frederic.godin@concordia.ca} (Fr\'ed\'eric Godin).  }}
\affil[a]{{\small Concordia University, Department of Mathematics and Statistics, Montr\'eal, Canada}}
\affil[b]{{\small Quantact Laboratory, Centre de Recherches Math\'ematiques, Montr\'eal, Canada}}

\vspace{-10pt}
\date{ 
\today}


\maketitle 

%


\begin{abstract}
\vspace{-5pt}
	
Estimating extreme risk in reinforcement learning becomes increasingly difficult as the confidence level increases, since only a small number of observations contribute to empirical estimates of Conditional Value-at-Risk (CVaR). In this paper, we extend an Extreme Value Theory (EVT) approach for CVaR optimization to high-dimensional neural-network policies. Using a peaks-over-threshold representation and method-of-moments estimation, we obtain a differentiable tail-risk objective and show that, conditional on the selected threshold, its finite-sample gradient can be computed directly through automatic differentiation. We first validate the approach in a controlled experiment with a known optimal solution and then investigate its performance in a deep-hedging problem with rare extreme market movements. Our results show similar hedging behavior for EVT- and empirical-CVaR-trained policies under normal conditions, while EVT-trained policies exhibit lower hedging losses under the extreme-scenario evaluation across the training sample sizes considered. These results suggest that EVT can provide a practical alternative for training neural policies when the observations relevant to extreme risk are scarce.
	
\bigskip 



\noindent \textbf{Keywords: }Extreme Value Theory; Reinforcement Learning; Policy Gradient; Tail Risk. 
\end{abstract}

\medskip



\doublespacing

\setcounter{page}{1}
\pagenumbering{arabic}


\section{Introduction}\label{se:intro}

Reinforcement learning (RL) provides a flexible framework for sequential decision problems in which actions are selected to optimize an objective over a multi-period horizon. In risk management applications, however, optimizing average performance may not adequately account for rare but potentially severe losses. Risk-aware RL addresses this problem by considering objectives such as the Conditional Value-at-Risk (CVaR), which explicitly focuses on the tail of the loss distribution. The estimation of such objectives becomes increasingly difficult at high confidence levels, since only a small fraction of the observations generated during training directly characterize the extreme tail.

Extreme Value Theory (EVT) provides a natural framework to address this scarcity of tail observations. Under the peaks-over-threshold (POT) approach, observations exceeding a sufficiently high threshold are modeled using a Generalized Pareto Distribution (GPD), allowing threshold exceedances to be used to extrapolate the behavior of the extreme tail and estimate extreme risk measures; see, e.g., \cite{troop2021bias}. Building on this idea, \cite{davar2025catastrophic} introduces an EVT-based policy-gradient methodology in which POT estimates of CVaR are incorporated into the optimization of risk-aware RL policies. Their numerical results show the potential benefits of the approach when observations in the extreme tail are scarce, particularly for heavy-tailed loss distributions.

The methodology of \cite{davar2025catastrophic}, however, relies on finite differences to approximate the gradient of the EVT-based objective with respect to the policy parameters. This requires repeatedly evaluating the objective after modifying individual policy parameters, making the procedure increasingly computationally restrictive as the number of policy parameters grows. This limitation concerns primarily the dimension of the policy parameterization rather than the dimension of the underlying sequential decision problem. Extending the POT framework to neural-network policies would therefore allow EVT-based tail-risk optimization to be combined with the flexibility offered by deep RL.

In this paper, we investigate this extension by constructing an EVT-based loss that can be differentiated with respect to the parameters of a neural-network policy. The proposed approach estimates the tail of the loss distribution using the POT methodology and propagates the resulting tail-risk objective through the policy using automatic differentiation. We complement the proposed methodology with a theoretical analysis of the differentiability and statistical properties of the EVT-based objective. In particular, the analysis distinguishes the local differentiability of the finite-sample EVT objective from the statistical consistency of the corresponding tail-risk estimator under explicit assumptions on the excess distribution.

We study the methodology numerically in two stages. We first consider a controlled experiment in which the optimal policy is analytically tractable. This setting allows us to compare the numerical behavior of the empirical and EVT-based optimization procedures against a known solution and to examine how their performance changes as the amount of information available for estimation increases. We then consider neural-network policies, using deep hedging as the principal application. Deep hedging provides a natural setting for this analysis since the terminal hedging error defines a loss distribution whose adverse tail is directly relevant for risk management, while the policy itself can be represented by a neural network \cite{buehler2019deep}.

Our analysis therefore extends EVT-based policy optimization from low-dimensional policy parameterizations relying on finite-difference gradients to neural-network policies trained through automatic differentiation. The theoretical analysis investigates the conditions supporting this extension, while the numerical experiments examine its behavior first in a controlled environment and subsequently in a deep-hedging setting. In the latter, we investigate whether explicitly modeling the tail during training provides an advantage over empirical CVaR optimization as losses become more extreme and tail observations increasingly scarce.

The remainder of the paper is organized as follows. Section~\ref{se:background} introduces the risk-aware reinforcement learning problem, the EVT framework for tail-risk estimation, and the EVT-based policy optimization approach on which our methodology builds. Section~\ref{se:methodology} presents the proposed extension to neural-network policies and the corresponding theoretical results. Section~\ref{se:numerical} presents the numerical study, beginning with a controlled experiment and subsequently considering the deep-hedging application. Finally, Section~\ref{se:conclusion} concludes.


\section{Risk-Aware Reinforcement Learning and Extreme Value Theory}
\label{se:background}

This section introduces the risk-aware reinforcement learning problem and the extreme value theory framework used to estimate losses in the far tail of the distribution. We then describe the peaks-over-threshold policy gradient approach introduced by \cite{davar2025catastrophic}, on which our methodology builds. The presentation is restricted to the mathematical objects required for the methodological extension developed in the following section.

\subsection{Risk-Aware Reinforcement Learning}
\label{se:risk_aware_rl}

Consider a finite-horizon sequential decision problem with time steps $\mathcal{T}=\{0,\ldots,T\}$, state space $\mathcal{S}$, action space $\mathcal{A}$, and cost space $\mathcal{C}$. At time $t$, the agent observes a state $S_t\in\mathcal{S}$ and selects an action $A_t\in\mathcal{A}$. Following the risk-aware reinforcement learning formulation considered by \cite{davar2025catastrophic}, we consider deterministic policies $\pi_\theta:\mathcal{S}\rightarrow\mathcal{A}$ parameterized by $\theta\in\Theta\subseteq\mathbb{R}^p$, such that $A_t=\pi_\theta(S_t)$. The policy therefore induces a random trajectory whose distribution, and consequently its associated costs, depends on $\theta$.

For a given policy $\pi_\theta$, let its cumulative discounted loss be
\begin{equation}
    X_\theta
    :=
    \sum_{t=1}^{T}\gamma^t C_t^{(\pi_\theta)},
    \qquad \gamma\in(0,1],
    \label{eq:policy_loss}
\end{equation}
where the superscript emphasizes the dependence of the costs on the policy. A risk-aware learning problem can then be formulated by selecting the policy parameters that minimize a risk functional $\rho$ of $X_\theta$,
\begin{equation*}
    \theta^\star
    \in
    \arg\min_{\theta\in\Theta}
    \rho(X_\theta).
\end{equation*}

In this work, our interest is in extreme losses, for which the upper tail of $X_\theta$ becomes particularly relevant. We therefore consider Conditional Value-at-Risk (CVaR), introduced as a risk measure for general loss distributions by \cite{rockafellar2002conditional}. Let $F_\theta$ denote the cumulative distribution function of $X_\theta$ and define its quantile at confidence level $\alpha\in(0,1)$ as
\begin{equation*}
    q_{\alpha,\theta}
    :=
    \inf\left\{
        x\in\mathbb{R}:F_\theta(x)\geq\alpha
    \right\}.
\end{equation*}
The corresponding CVaR is given by
\begin{equation}
    \operatorname{CVaR}_{\alpha}(X_\theta)
    =
    \frac{1}{1-\alpha}
    \int_{\alpha}^{1}
    q_{\beta,\theta}\,d\beta.
    \label{eq:cvar_definition}
\end{equation}
When $X_\theta$ is absolutely continuous, \autoref{eq:cvar_definition} can equivalently be expressed as $\mathbb{E}[X_\theta\mid X_\theta\geq q_{\alpha,\theta}]$. Hence, CVaR represents the expected loss conditional on being within the worst $100(1-\alpha)\%$ of the loss distribution.

The risk-aware objective considered throughout this paper is consequently
\begin{equation}
    J_\alpha(\theta)
    :=
    \operatorname{CVaR}_{\alpha}(X_\theta),
    \qquad
    \theta^\star_\alpha
    \in
    \arg\min_{\theta\in\Theta}J_\alpha(\theta).
    \label{eq:cvar_policy_objective}
\end{equation}
For values of $\alpha$ close to one, however, estimating \autoref{eq:cvar_policy_objective} directly from the empirical distribution becomes challenging because only a small fraction of the observations contributes to the estimation of the far tail. This problem becomes particularly relevant when extreme observations are scarce or when the loss distribution is heavy-tailed \cite{davar2025catastrophic}. Extreme value theory provides a natural framework to address this problem by using less extreme observations to characterize and extrapolate the behavior of the upper tail.

\subsection{Extreme Value Theory and Tail-Risk Estimation}
\label{se:evt}

A classical approach in extreme value theory is the peaks-over-threshold (POT) method, where observations exceeding a sufficiently high threshold $u$ are used to characterize the upper tail of a distribution; see \cite{coles2001introduction,mcneil2015quantitative}. In our setting, the objective is to use these exceedances to estimate the region of the loss distribution associated with $q_{\alpha,\theta}$, even when relatively few observations reach the target confidence level $\alpha$.

For a threshold $u$, define the conditional excess distribution of the policy loss as
\begin{equation}
    F_{\theta,u}(y)
    :=
    \mathbb{P}
    \left(
        X_\theta-u\leq y
        \mid X_\theta>u
    \right)
    =
    \frac{
        F_\theta(u+y)-F_\theta(u)
    }{
        1-F_\theta(u)
    },
    \qquad y\geq0.
    \label{eq:excess_distribution}
\end{equation}

The theoretical motivation for modeling \autoref{eq:excess_distribution} follows from the asymptotic results of \cite{balkema1974residual} and \cite{pickands1975statistical}. More precisely, if $F_\theta$ belongs to the maximum domain of attraction of a generalized extreme value distribution, the distribution of the exceedances above a sufficiently high threshold can be approximated by a Generalized Pareto Distribution (GPD).

The GPD with shape parameter $\xi$ and scale parameter $\sigma>0$ has cumulative distribution function
\begin{equation}
    G_{\xi,\sigma}(y)
    =
    \begin{cases}
        1-
        \left(
            1+\dfrac{\xi y}{\sigma}
        \right)^{-1/\xi},
        & \xi\neq0, \\[2mm]
        1-\exp\left(-\dfrac{y}{\sigma}\right),
        & \xi=0,
    \end{cases}
    \label{eq:gpd}
\end{equation}
defined on the support satisfying $1+\xi y/\sigma>0$. The parameter $\xi$ determines the tail behavior of the distribution, whereas $\sigma$ controls its scale \cite{coles2001introduction,mcneil2015quantitative}. For a sufficiently high $u$, this gives the approximation
\begin{equation*}
    F_{\theta,u}(y)
    \approx
    G_{\xi_\theta(u),\sigma_\theta(u)}(y).
\end{equation*}
Notice that $\xi_\theta(u)$ and $\sigma_\theta(u)$ characterize the tail induced by the policy and the selected threshold, and therefore do not depend on the target confidence level $\alpha$.

An important consequence of this approximation is that it provides a mechanism to estimate CVaR beyond the selected threshold. Following the POT representation described by \cite{mcneil2015quantitative} and used for policy optimization by \cite{davar2025catastrophic}, define
\begin{equation}
    s_{\theta,\alpha}(u)
    :=
    \frac{1-F_\theta(u)}{1-\alpha}.
    \label{eq:tail_ratio}
\end{equation}
Provided that $q_{\alpha,\theta}>u$ and $\xi_\theta(u)<1$, the resulting POT approximation of $J_\alpha(\theta)$ is
\begin{equation}
    \mathcal{C}_{\theta,\alpha}(u)
    :=
    \begin{cases}
    \displaystyle
    u+
    \frac{\sigma_\theta(u)}{1-\xi_\theta(u)}
    \left[
        1+
        \frac{
            s_{\theta,\alpha}(u)^{\xi_\theta(u)}-1
        }{
            \xi_\theta(u)
        }
    \right],
    & \xi_\theta(u)\neq0, \\[4mm]
    \displaystyle
    u+
    \sigma_\theta(u)
    \left[
        1+\log s_{\theta,\alpha}(u)
    \right],
    & \xi_\theta(u)=0.
    \end{cases}
    \label{eq:evt_cvar_functional}
\end{equation}
Thus, $J_\alpha(\theta)\approx\mathcal{C}_{\theta,\alpha}(u)$, where the quality of the approximation depends on how well the GPD represents the excess distribution above $u$. Moreover, the condition $\xi_\theta(u)<1$ guarantees that the corresponding tail expectation is finite.

In practice, let $X_{\theta,1},\ldots,X_{\theta,n}$ denote independent realizations of $X_\theta$ and define
\begin{equation*}
    \mathcal{I}_u
    =
    \{i:X_{\theta,i}>u\},
    \qquad
    k=|\mathcal{I}_u|,
    \qquad
    Y_i=X_{\theta,i}-u,
    \quad i\in\mathcal{I}_u.
\end{equation*}
The GPD parameters can then be estimated from the exceedances $\{Y_i\}_{i\in\mathcal{I}_u}$. While different estimation procedures can be considered, including maximum likelihood, our subsequent formulation relies on the method-of-moments estimator described in \cite{davar2025catastrophic},
\begin{equation}
    \widehat{\xi}
    =
    \frac{S_Y^2-\overline{Y}^{\,2}}{2S_Y^2},
    \qquad
    \widehat{\sigma}
    =
    \overline{Y}
    \left(
        \frac{S_Y^2+\overline{Y}^{\,2}}{2S_Y^2}
    \right),
    \label{eq:gpd_mom}
\end{equation}
where
\begin{equation*}
    \overline{Y}
    =
    \frac{1}{k}\sum_{i\in\mathcal{I}_u}Y_i,
    \qquad
    S_Y^2
    =
    \frac{1}{k}
    \sum_{i\in\mathcal{I}_u}
    (Y_i-\overline{Y})^2.
\end{equation*}
Since this estimator matches the first two moments of the GPD, its moment representation requires $\xi<1/2$. In the numerical implementation, the fitted shape parameter is additionally restricted to the admissible region used by the training algorithm; the exact restriction is stated in
Appendix~\ref{appen:EVTtraining}. This implementation detail is kept separate from the population GPD approximation in \autoref{eq:evt_cvar_functional}.

Replacing $\xi_\theta(u)$, $\sigma_\theta(u)$, and $F_\theta(u)$ in \autoref{eq:evt_cvar_functional} by their corresponding sample estimates gives the EVT-based estimator
\begin{equation}
    \widehat{J}^{\mathrm{EVT}}_{n,\alpha}(\theta;u)
    :=
    \widehat{\mathcal{C}}_{\theta,\alpha}(u).
    \label{eq:evt_cvar_estimator}
\end{equation}
The choice of $u$ is therefore important. A low threshold increases the number of exceedances available for estimation, but the GPD approximation may not yet accurately represent the tail. Conversely, increasing $u$ improves the asymptotic motivation of the approximation while reducing the number of observations available to estimate its parameters. 

Since the threshold must be repeatedly determined during policy learning, manual threshold selection is not convenient in this setting. To address this problem, \cite{davar2025catastrophic} adopt the automated procedure proposed by \cite{bader2018automated}, based on a sequence of goodness-of-fit tests, together with the implementation developed by \cite{troop2021bias}. This provides an automated mechanism for determining the region from which the GPD approximation is used during policy optimization.

\subsection{EVT-Based Policy Optimization}
\label{se:evt_policy_optimization}

The connection between EVT-based tail estimation and reinforcement learning was introduced by \cite{davar2025catastrophic} through the peaks-over-threshold policy gradient (POTPG) algorithm. The main idea is to replace the direct empirical estimate of CVaR by the POT estimator in \autoref{eq:evt_cvar_estimator} during policy optimization.

At a given parameter value $\theta$, independent trajectories are generated under $\pi_\theta$, producing a sample $X_{\theta,1},\ldots,X_{\theta,n}$ of policy losses. A threshold $u$ is selected, the GPD parameters are estimated from the corresponding exceedances, and the resulting POT approximation is used to estimate the objective in \autoref{eq:cvar_policy_objective}. The corresponding sample optimization problem can therefore be represented as
\begin{equation}
    \widehat{\theta}^{\,\star}_{n,\alpha}
    \in
    \arg\min_{\theta\in\Theta}
    \widehat{J}^{\mathrm{EVT}}_{n,\alpha}(\theta;u).
    \label{eq:evt_policy_objective}
\end{equation}
In this way, policy learning does not rely exclusively on the relatively small number of observations beyond $q_{\alpha,\theta}$. Instead, observations exceeding $u$ are used to model the upper tail and extrapolate toward the region associated with the target confidence level.

To optimize \autoref{eq:evt_policy_objective}, \cite{davar2025catastrophic} estimate the gradient coordinatewise using forward finite differences. For $\theta=(\theta_1,\ldots,\theta_p)^\top$ and $\varepsilon>0$,
\begin{equation}
    \widehat{\partial_i J_\alpha}(\theta)
    =
    \frac{
        \widehat{J}^{\mathrm{EVT}}_{n,\alpha}
        (\theta+\varepsilon e_i;u)
        -
        \widehat{J}^{\mathrm{EVT}}_{n,\alpha}
        (\theta;u)
    }{
        \varepsilon
    },
    \qquad i=1,\ldots,p,
    \label{eq:potpg_finite_difference}
\end{equation}
where $e_i$ denotes the $i$th canonical basis vector. The resulting gradient estimate can then be incorporated into a stochastic optimization procedure to update the policy parameters.

An important aspect of POTPG is that the threshold selected under the current policy is kept fixed when evaluating the coordinatewise perturbations in \autoref{eq:potpg_finite_difference}. Hence, each perturbed policy changes the estimated tail distribution while the threshold used to define the exceedance region remains unchanged within the corresponding gradient evaluation.

This construction provides a direct connection between EVT-based tail estimation and risk-aware policy learning. However, \autoref{eq:potpg_finite_difference} requires an additional evaluation of the EVT objective for every coordinate of $\theta$. Therefore, as the dimension $p$ of the policy parameterization increases, so does the number of objective evaluations required to construct a single gradient. This becomes particularly restrictive when $\pi_\theta$ is represented by a neural network with a large number of parameters.

The following section builds on the same POT representation of the risk objective, but reformulates its optimization to exploit the differentiable structure of neural-network-parameterized policies.


\section{EVT-Based Optimization of Neural-Network Policies}
\label{se:methodology}

The EVT-based policy optimization framework introduced in
\autoref{se:evt_policy_optimization} relies on the finite-difference
approximation in~\autoref{eq:potpg_finite_difference}. While suitable for
low-dimensional policy parameterizations, its computational cost increases
with the number of policy parameters, limiting its direct application to
neural-network policies.

We address this limitation by formulating the POT-MOM estimator in
\autoref{eq:evt_cvar_estimator} as a differentiable empirical objective.
This allows its gradient with respect to all policy parameters to be
computed jointly using automatic differentiation, without modifying the
underlying EVT approximation. In addition, when the threshold-selection
procedure does not identify a threshold for which the GPD hypothesis is not
rejected, the method falls back to empirical CVaR.

\subsection{Differentiable POT-CVaR Objective}
\label{se:pot_cvar}

Consider the EVT-CVaR estimator
$\widehat J^{\mathrm{EVT}}_{n,\alpha}(\theta;u)$ defined in
\autoref{eq:evt_cvar_estimator}. For a fixed threshold $u$, the exceedances,
their empirical moments, and consequently the MOM estimates in
\autoref{eq:gpd_mom} depend on the policy parameters through the sampled
losses $X_{\theta,i}$. We preserve this dependence when evaluating the
estimator, resulting in the mapping
\begin{equation}
    \theta
    \longmapsto
    \widehat J^{\mathrm{EVT}}_{n,\alpha}(\theta;u).
    \label{eq:differentiable_evt_objective}
\end{equation}
Provided that the sampled losses are differentiable with respect to
$\theta$ and that the exceedance set remains locally unchanged,
\autoref{eq:differentiable_evt_objective} can therefore be differentiated
directly with respect to the policy parameters.

The threshold-selection step remains statistical and is not differentiated.
Let $A_n(\theta)$ denote the event that the procedure described in
\autoref{se:evt_policy_optimization} identifies an admissible threshold
$\widehat u_n(\theta)$ for which the GPD hypothesis is not rejected. The
training objective is then
\begin{equation}
    \widehat J^{\mathrm{H}}_{n,\alpha}(\theta)
    =
    \begin{cases}
        \widehat J^{\mathrm{EVT}}_{n,\alpha}
        \bigl(\theta;\widehat u_n(\theta)\bigr),
        & A_n(\theta), \\[1mm]
        \widehat J^{\mathrm{emp}}_{n,\alpha}(\theta),
        & A_n^c(\theta),
    \end{cases}
    \label{eq:hybrid_evt_objective}
\end{equation}
where $\widehat J^{\mathrm{emp}}_{n,\alpha}$ denotes the empirical CVaR
estimator
\begin{equation}
    \widehat J^{\mathrm{emp}}_{n,\alpha}(\theta)
    :=
    \widehat q_{\alpha,\theta}
    +
    \frac{1}{n(1-\alpha)}
    \sum_{i=1}^{n}
    \left(X_{\theta,i}-\widehat q_{\alpha,\theta}\right)_{+},
    \label{eq:empirical_cvar_estimator}
\end{equation}
where $\widehat q_{\alpha,\theta}$ denotes the empirical
$\alpha$-quantile of the sampled losses. When the EVT branch is
selected, $\widehat u_n(\theta)$ is held fixed during differentiation.

Importantly, the event $A_n(\theta)$ only determines whether the POT
approximation is used. Failure to reject the GPD hypothesis does not imply
that the population excess distribution is exactly GPD.

\subsection{Neural-Network Policy Optimization}
\label{se:nn_optimization}

We now parameterize the policy $\pi_\theta$ by a neural network, with
$\theta\in\Theta\subseteq\mathbb R^p$ containing its trainable parameters.
The population objective remains the CVaR minimization problem in
\autoref{eq:cvar_policy_objective}; only the mechanism used to estimate its
gradient is modified.

At iteration $m$, trajectories are generated under $\pi_{\theta_m}$ and the
threshold-selection procedure is applied to the resulting losses. If the
EVT branch in \autoref{eq:hybrid_evt_objective} is selected, we compute
\begin{equation}
    g^{\mathrm{EVT}}_{n,\alpha}(\theta_m;u_m)
    :=
    \left.
    \nabla_\theta
    \widehat J^{\mathrm{EVT}}_{n,\alpha}(\theta;u_m)
    \right|_{\theta=\theta_m},
    \label{eq:evt_autodiff_gradient}
\end{equation}
where $u_m$ is held fixed. Automatic differentiation propagates through the
sampled losses and the POT-MOM estimator, but not through the
threshold-selection procedure.

This differs from the finite-difference estimator in
\autoref{eq:potpg_finite_difference}, which requires a separate perturbed
evaluation for each component of $\theta$. Here, all components of
\autoref{eq:evt_autodiff_gradient} are obtained from the same differentiable
objective, allowing the EVT framework to be applied when $p$ corresponds
to the number of parameters of a neural network.

When no admissible threshold is identified, the gradient is instead
computed from the empirical CVaR branch of
\autoref{eq:hybrid_evt_objective}. Denoting the gradient of the active
branch by $g^{\mathrm{H}}_{n,\alpha}(\theta_m)$, the policy parameters are
updated through the corresponding adaptive-gradient update used in the numerical implementation or through
\begin{equation}
    \theta_{m+1}
    =
    \theta_m
    -
    \eta_m g^{\mathrm{H}}_{n,\alpha}(\theta_m),
    \label{eq:evt_nn_update}
\end{equation}
where $\eta_m>0$ denotes the learning rate at iteration $m$.

\subsection{Theoretical Results}
\label{se:theory}

We first establish that automatic differentiation provides the exact
gradient of the finite-sample POT-MOM objective when the EVT branch is
active. The result is local because both the exceedance set and the selected
threshold can change as the policy parameters vary.

\begin{assumption}[Local regularity]
\label{ass:empirical_differentiability}
Fix a sample, a threshold $u$, and $\theta_0\in\Theta$. Assume that, in a
neighborhood of $\theta_0$:
(i) each sampled loss $X_{\theta,i}$ is continuously differentiable in
$\theta$;
(ii) no sampled loss is equal to $u$;
(iii) the empirical moments required by \autoref{eq:gpd_mom} are
non-degenerate; and
(iv) the resulting MOM estimates satisfy
$\widehat{\sigma}>0$ and $\widehat{\xi}<1/2$.
The threshold $u$ is held fixed during differentiation.
\end{assumption}

The restriction $\widehat{\xi}<1/2$ retains the estimator within the
parameter region associated with the finite-second-moment condition required
by the MOM construction \cite{davar2025catastrophic}.

\begin{proposition}[Local differentiability of the POT-MOM objective]
\label{prop:empirical_pot_differentiability}
Under Assumption~\ref{ass:empirical_differentiability},
$\widehat J^{\mathrm{EVT}}_{n,\alpha}(\theta;u)$ is continuously
differentiable in a neighborhood of $\theta_0$, and automatic
differentiation returns
\begin{equation}
    g^{\mathrm{EVT}}_{n,\alpha}(\theta_0;u)
    =
    \left.
    \nabla_\theta
    \widehat J^{\mathrm{EVT}}_{n,\alpha}(\theta;u)
    \right|_{\theta=\theta_0}.
    \label{eq:exact_autodiff_gradient}
\end{equation}
\end{proposition}

\emph{Proof.}
Since no sampled loss lies on $u$, continuity implies that the exceedance
set defined in \autoref{se:evt} is locally constant. The empirical moments
of the excesses are therefore continuously differentiable functions of
$\theta$. The MOM estimators in \autoref{eq:gpd_mom} are differentiable
functions of these moments whenever their denominators are non-zero, and
the EVT-CVaR estimator in \autoref{eq:evt_cvar_estimator} is differentiable
over the admissible parameter region. The result follows by composition
and the chain rule. \hfill$\square$

Proposition~\ref{prop:empirical_pot_differentiability} concerns the active
POT branch with its selected threshold held fixed. It does not establish
differentiability through the threshold-selection procedure, nor does it
imply consistency of the resulting gradient for
$\nabla_\theta J_\alpha(\theta)$.

We next consider the population behavior of the POT-MOM estimator. Recall
the POT-CVaR functional $\mathcal C_{\theta,\alpha}(u)$ introduced in
\autoref{eq:evt_cvar_functional}.

\begin{assumption}[POT-MOM estimation]
\label{ass:pot_mom_limit}
For fixed $\theta$ and an admissible threshold $u<q_{\alpha,\theta}$,
assume that the conditional excess distribution above $u$ is GPD with
parameters $\xi_\theta(u)$ and $\sigma_\theta(u)$ satisfying
\[
    \sigma_\theta(u)>0,
    \qquad
    \xi_\theta(u)<\frac{1}{2},
\]
and that the number of exceedances satisfies
\[
    k_n(\theta;u)\xrightarrow{\mathbb P}\infty.
\]
\end{assumption}

\begin{proposition}[Consistency of the POT-MOM estimator]
\label{prop:pot_mom_consistency}
Under Assumption~\ref{ass:pot_mom_limit},
\begin{equation}
    \widehat J^{\mathrm{EVT}}_{n,\alpha}(\theta;u)
    \xrightarrow{\mathbb P}
    \mathcal C_{\theta,\alpha}(u).
    \label{eq:evt_objective_consistency}
\end{equation}
\end{proposition}

\emph{Proof.}
Under $\xi_\theta(u)<1/2$, the first two moments of the GPD excess
distribution are finite. As $k_n(\theta;u)\rightarrow\infty$, their empirical
counterparts converge to the corresponding population moments. The MOM
estimates in \autoref{eq:gpd_mom} therefore converge, by the continuous
mapping theorem, to $\xi_\theta(u)$ and $\sigma_\theta(u)$. Moreover, the
empirical exceedance probability converges to $1-F_\theta(u)$. Applying the
continuous mapping theorem to the EVT-CVaR estimator in
\autoref{eq:evt_cvar_estimator} gives
\autoref{eq:evt_objective_consistency}. \hfill$\square$

Under Assumption~\ref{ass:pot_mom_limit}, the GPD representation is exact
above $u$. Since $u<q_{\alpha,\theta}$, the target quantile lies within this
tail region and the POT representation in
\autoref{eq:evt_cvar_functional} satisfies
\begin{equation}
    \mathcal C_{\theta,\alpha}(u)
    =
    J_\alpha(\theta).
    \label{eq:exact_tail_identity}
\end{equation}
Consequently,
\begin{equation}
    \widehat J^{\mathrm{EVT}}_{n,\alpha}(\theta;u)
    \xrightarrow{\mathbb P}
    J_\alpha(\theta).
    \label{eq:evt_true_cvar_consistency}
\end{equation}

The exact-GPD condition in Assumption~\ref{ass:pot_mom_limit} is a
population assumption and should be distinguished from the statistical
criterion used to select the EVT branch. In particular, failure to reject
the GPD hypothesis does not establish that the population excess
distribution is exactly GPD. The result therefore characterizes the
population behavior of the POT-MOM estimator under an exact-tail model,
rather than providing a consequence of the goodness-of-fit test itself.

We finally consider the effect of the empirical-CVaR fallback. Recall that
$A_n(\theta)$ denotes the event that an admissible threshold is selected and
the EVT branch of \autoref{eq:hybrid_evt_objective} is used.

\begin{corollary}[Consistency of the hybrid estimator]
\label{cor:hybrid_consistency}
Suppose that, for fixed $\theta$,
\begin{equation}
    \mathbf 1_{A_n(\theta)}
    \left|
        \widehat J^{\mathrm{EVT}}_{n,\alpha}
        \bigl(\theta;\widehat u_n(\theta)\bigr)
        -
        J_\alpha(\theta)
    \right|
    \xrightarrow{\mathbb P}0,
    \label{eq:selected_evt_consistency}
\end{equation}
and
\begin{equation}
    \mathbf 1_{A_n^c(\theta)}
    \left|
        \widehat J^{\mathrm{emp}}_{n,\alpha}(\theta)
        -
        J_\alpha(\theta)
    \right|
    \xrightarrow{\mathbb P}0.
    \label{eq:selected_empirical_consistency}
\end{equation}
Then the hybrid estimator satisfies
\begin{equation}
    \widehat J^{\mathrm{H}}_{n,\alpha}(\theta)
    \xrightarrow{\mathbb P}
    J_\alpha(\theta).
    \label{eq:hybrid_consistency}
\end{equation}
\end{corollary}

\emph{Proof.}
From \autoref{eq:hybrid_evt_objective},
\begin{align*}
    \left|
        \widehat J^{\mathrm{H}}_{n,\alpha}(\theta)
        -
        J_\alpha(\theta)
    \right|
    ={}&
    \mathbf 1_{A_n(\theta)}
    \left|
        \widehat J^{\mathrm{EVT}}_{n,\alpha}
        \bigl(\theta;\widehat u_n(\theta)\bigr)
        -
        J_\alpha(\theta)
    \right|
    \\
    &+
    \mathbf 1_{A_n^c(\theta)}
    \left|
        \widehat J^{\mathrm{emp}}_{n,\alpha}(\theta)
        -
        J_\alpha(\theta)
    \right|.
\end{align*}
Both terms converge to zero in probability by
\autoref{eq:selected_evt_consistency} and
\autoref{eq:selected_empirical_consistency}, which proves
\autoref{eq:hybrid_consistency}. \hfill$\square$

Corollary~\ref{cor:hybrid_consistency} does not require the probability of
selecting the EVT branch to converge to one. Instead, it requires the
estimator used on each branch to approach the same population target.
For the empirical branch, this corresponds to consistency of the empirical
CVaR estimator. For the EVT branch, however,
\autoref{eq:selected_evt_consistency} is an additional condition because
the threshold $\widehat u_n(\theta)$ is selected from the data. The results therefore establish local differentiability of the active
POT-MOM objective and consistency of the risk estimator under the stated
tail conditions.


\section{Numerical Study}
\label{se:numerical}

We now investigate the numerical behavior of the methodology developed in Section~\ref{se:methodology}. We first consider a controlled experiment for which the optimal policy is known analytically, and subsequently move to a deep-hedging problem where the policy is represented by a neural network and catastrophic losses arise from rare market scenarios.

\subsection{Controlled Experiment}
\label{se:controlled_experiment}

We begin with a controlled experiment to examine whether the proposed optimization procedure can recover a known solution. Following the setting considered in \cite{davar2025catastrophic}, let the loss $X_{\boldsymbol{\theta}}$ follow a generalized Pareto distribution with shape parameter $\xi=0.4$ and scale parameter depending on the two-dimensional policy parameter $\boldsymbol{\theta}=(\theta_1,\theta_2)$ through
\begin{equation}
    \sigma(\boldsymbol{\theta})
    =
    \log\left(1+(\theta_1-2)^2+(\theta_2+1)^2\right)+2.
    \label{eq:controlled_scale}
\end{equation}
In this controlled setting, the policy directly determines $\theta_1$ and $\theta_2$, which in turn determine the distribution of the terminal loss through $\sigma(\boldsymbol{\theta})$. Simulated realizations of $X_{\boldsymbol{\theta}}$ therefore play the role of the terminal outcomes generated under the policy, and the optimization problem reduces to finding the policy parameters that minimize their tail risk. In particular, for $\xi\neq 0$ and $\xi<1$, the CVaR of the generalized Pareto loss can be expressed as
\begin{equation}
    \operatorname{CVaR}_{\alpha}(X_{\boldsymbol{\theta}})
    =
    \frac{\sigma(\boldsymbol{\theta})}{1-\xi}
    \left[
        \frac{(1-\alpha)^{-\xi}-1}{\xi}
        +1
    \right],
    \label{eq:controlled_cvar}
\end{equation}
so that the dependence of the objective on the policy parameters is entirely governed by $\sigma(\boldsymbol{\theta})$. Since the multiplicative dependence on $\sigma(\boldsymbol{\theta})$ is positive for fixed $\alpha$ and $\xi<1$, minimizing $\operatorname{CVaR}_{\alpha}(X_{\boldsymbol{\theta}})$ is equivalent to minimizing $\sigma(\boldsymbol{\theta})$. The optimal policy is therefore available analytically and satisfies
\begin{equation}
    \boldsymbol{\theta}^{*}=(2,-1).
    \label{eq:controlled_optimum}
\end{equation}
This provides a direct benchmark against which the numerical optimization can be evaluated.

Starting from $\boldsymbol{\theta}_0=(4,-3)$, we minimize $\operatorname{CVaR}_{0.99}$ using $2{,}000$ simulated observations per iteration and compare the empirical-CVaR and EVT-based objectives. Both procedures progressively move the policy parameters toward the analytical optimum. After $100$ iterations, the EVT-based procedure reaches approximately $(1.99,-0.99)$, which is close to $\boldsymbol{\theta}^{*}$, see \autoref{fig:control_example}. The experiment therefore provides controlled numerical evidence that the differentiable EVT objective can guide gradient-based optimization toward the correct policy in this tractable setting. This observation is complementary to, but distinct from, the theoretical results of Section~\ref{se:theory}.

\begin{figure}[H]\centering
    \caption{Optimization trajectories obtained using the empirical and EVT-based CVaR estimators for the controlled experiment.}
\begin{threeparttable}
    \includegraphics[width=16cm]{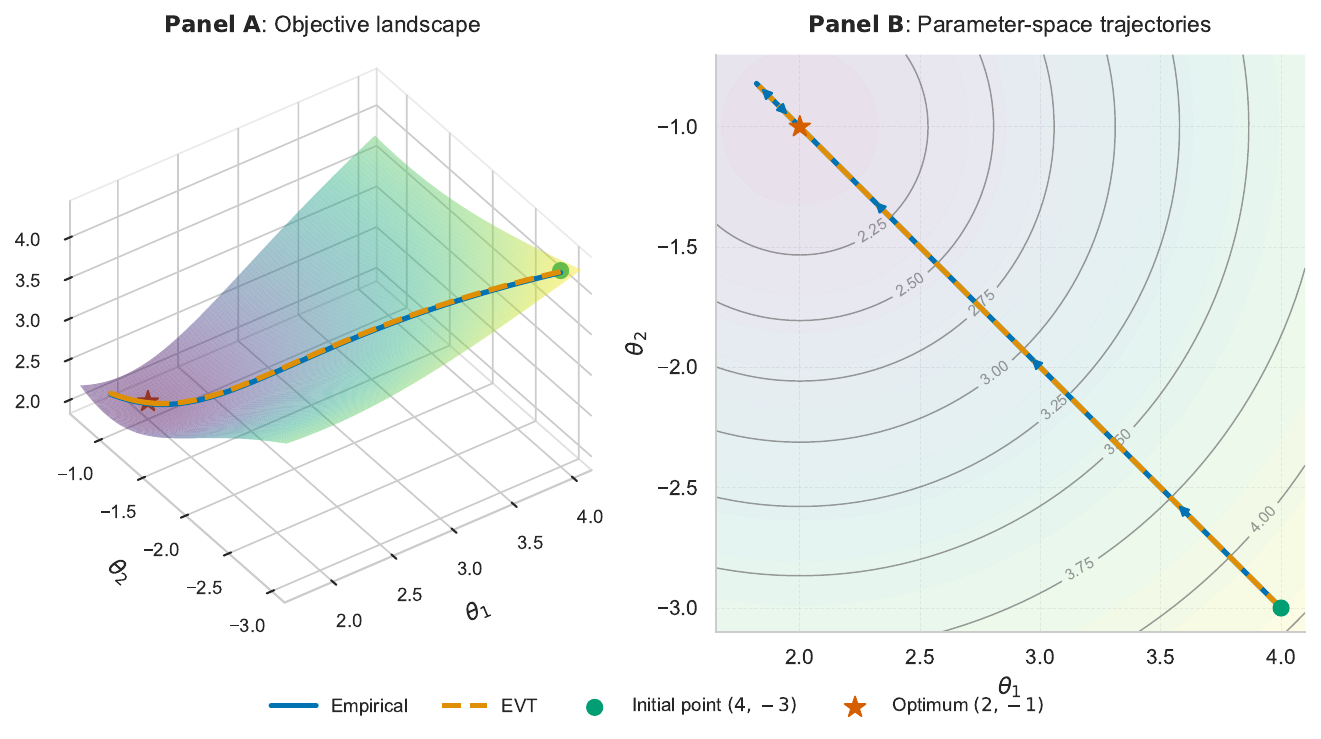}
    \begin{tablenotes}
    \item \footnotesize Panel A illustrates the optimization trajectories over the parameter-dependent scale function, while Panel B presents the corresponding trajectories in the parameter space. Both procedures are initialized at $(\theta_1,\theta_2)=(4,-3)$, with the theoretical minimum located at $(2,-1)$. The empirical and EVT-based approaches follow similar trajectories toward the minimum, initially overshooting it before stabilizing in its neighborhood.
    \end{tablenotes}
\end{threeparttable}
    \label{fig:control_example}
\end{figure}
\subsection{Deep Hedging Framework}
\label{se:deep_hedging}

We next consider a dynamic option-hedging problem, which provides a higher-dimensional policy parameterization and a direct financial interpretation of extreme losses. Consistently with the deep-hedging formulation, the agent determines at each rebalancing date the position in the underlying asset as a function of the current state, with the policy represented by the neural network introduced in Section~\ref{se:nn_optimization}. The terminal hedging error constitutes the loss entering the risk objective.

The underlying asset follows a Merton-type jump-diffusion process,
\begin{equation}
    \frac{dS_t}{S_{t^-}}
    =
    \mu\,dt+\sigma\,dW_t+
    \left(e^{J_t}-1\right)dN_t,
    \label{eq:jump_diffusion}
\end{equation}
where $W_t$ is a Brownian motion, $N_t$ is a Poisson process with intensity $\lambda$, and the log-jump sizes are Gaussian with mean $\nu$ and standard deviation $\delta$. Under the risk-neutral measure, the drift is adjusted by $\lambda\left(e^{\nu+\delta^2/2}-1\right)$ when computing the initial option value.

The main experiment considers $S_0=100$, a European call with strike $K=140$ and $63$ daily rebalancing periods. The annual parameters are $\mu=0.01$, $\sigma=0.20$, $\lambda=0.1$, $\nu=0.4$, $\delta=0.2$, and $r=0.005$. The option is therefore initially out-of-the-money, while the positive jumps generate infrequent but potentially large terminal payoffs. This setting is particularly useful for the present study: the loss distribution contains a small number of observations far into the upper tail, precisely the regime in which direct empirical estimation of catastrophic-level CVaR becomes data demanding.

The hedging policy is represented by a feedforward neural network whose state space is restricted to the log-price of the underlying asset, time-to-maturity, and the previous hedging position. The network consists of four hidden layers with $56$ neurons each and a single output representing the position in the underlying asset, following the standard architecture considered for asymmetric penalty measures in \cite{François2025_stat_arbitrage}. The output layer incorporates the dynamic activation function proposed in \cite{François2025}, with leverage constraint $B=100$. This constraint prevents the agent from adopting doubling strategies, whereby increasingly large positions are taken following losses in an attempt to recover previous hedging errors, which can lead to unrealistic policies and distort the assessment of tail-risk performance. The network is trained using the ADAM optimizer with an initial learning rate of $0.0005$ and dropout regularization with probability $p=0.5$. The empirical-CVaR and EVT-based policies use the same state representation, network architecture, leverage constraint, optimizer, learning rate, dropout specification, training paths, and epoch budget, such that the tail-risk objective is the experimental factor being changed. Both policies are trained to minimize CVaR at the confidence level $\alpha=0.999$.

\subsection{Neural-Network Experiments}
\label{se:nn_experiments}

We first investigate whether replacing empirical CVaR with the EVT-based objective affects the training behavior of neural-network policies when extreme observations are not especially prominent. To account for the stochasticity of neural-network optimization, we train otherwise identical feedforward policies over four independent runs and compare their validation CVaR throughout $50$ training epochs. The resulting trajectories are presented in \autoref{fig:nn_training}.

A clear difference appears during the initial stages of training. The EVT-based objective exhibits a faster reduction in mean validation CVaR, while the empirical approach decreases more gradually. This difference is particularly pronounced during approximately the first $15$ epochs. The dispersion across independent runs is also considerably smaller for the EVT-based approach during this period, whereas the empirical-CVaR policies exhibit greater variability in their validation trajectories. After approximately $20$ epochs, however, the mean validation CVaR of both approaches becomes similar and remains close throughout the remainder of training. This suggests that the EVT-based objective induces more stable optimization behavior during the early training stage, which is
consistent with the policies reaching the region of their eventual validation performance with fewer training iterations.

\begin{figure}[h]\centering
    \caption{Validation $\mathrm{CVaR}_{0.999}$ throughout
training under the empirical and EVT-based optimization approaches.}
\begin{threeparttable}
    \includegraphics[width=14cm]{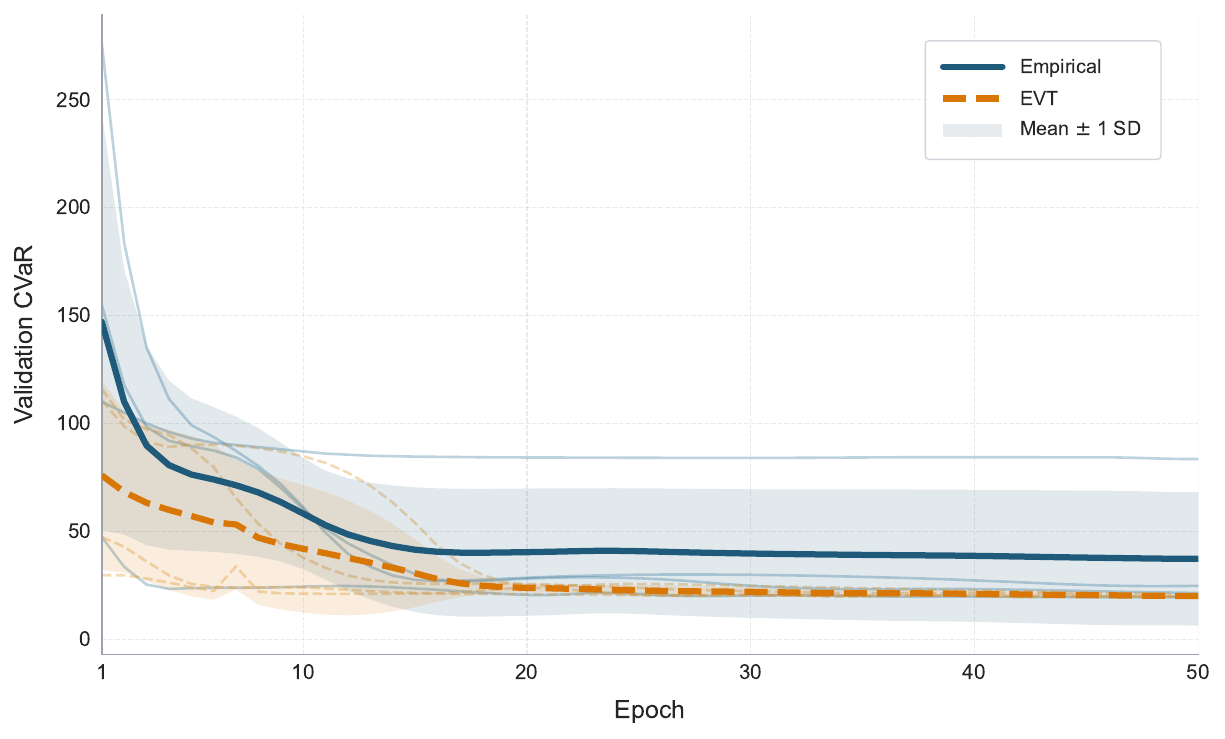}
    \begin{tablenotes}
    \item \footnotesize Validation $\mathrm{CVaR}_{0.999}$ is reported over $50$ training epochs for the empirical and EVT-based approaches. The solid and dashed curves represent the mean validation $\mathrm{CVaR}_{0.999}$ across four independent runs for the empirical and EVT-based approaches, respectively, while the shaded regions indicate one standard deviation around the mean. The lighter curves correspond to the individual runs.
    \end{tablenotes}
\end{threeparttable}
    \label{fig:nn_training}
\end{figure}

The experiment therefore indicates that the main difference between the two objectives in this setting concerns the dynamics and stability of the optimization rather than the validation risk eventually attained. This behavior is consistent with the statistical motivation of the proposed estimator: when relatively few observations determine the empirical tail-risk objective, variation in those observations can propagate directly into the optimization step, whereas the POT approximation uses the exceedance distribution to estimate the relevant tail quantity. Once training progresses and both policies reach a similar region of the parameter space, the difference between their validation performance becomes substantially smaller. This raises a more demanding question for the proposed methodology: whether the difference between the two objectives becomes more consequential when policy performance is evaluated specifically in the extreme tail of the loss distribution.

\subsection{Extreme-Loss Analysis}
\label{se:extreme_loss}

The previous experiment suggests that the empirical and EVT-based objectives eventually attain similar validation risk under the general training setting. We now investigate whether this similarity persists when the learned policies are evaluated specifically under extreme market realizations. Both policies are trained using paths generated from the Merton jump-diffusion model described in \autoref{se:deep_hedging}.

To define the extreme scenarios independently of the learned policies, let
$$
H_T := (S_T-K)^+
$$
denote the unhedged terminal option payoff, and let $q_p^H$ denote its $p$-quantile under the original market distribution. We define the extreme-event set as
$$
\mathcal{E}_p := \left\{H_T > q_p^H\right\}.
$$
For the present experiment, we set $p=0.999$, such that $\mathcal{E}_{0.999}$ corresponds to market realizations in the upper $0.1\%$ tail of the unhedged option-payoff distribution.

We consider two independent out-of-sample evaluation samples. The first consists of 100,000 paths generated directly from the original market distribution. The second is constructed by independently simulating additional paths from the same market model and retaining complete trajectories belonging to $\mathcal{E}_{0.999}$ until 3,000 such trajectories are obtained. Thus, the second sample concentrates the evaluation on rare and severe terminal market outcomes while leaving the underlying data-generating process unchanged. For brevity, we refer to this conditional extreme-event sample as the heavy-tail regime throughout the remainder of the numerical study. This construction allows us to compare the behavior of the learned policies under the general market distribution with their behavior conditional on extreme market realizations.

\autoref{fig:extreme_loss} shows the resulting out-of-sample hedging-error distributions. Under the general distribution in Panel A, the two policies exhibit very similar behavior. Both have a mean hedging error of $-1.31$, while their medians are $-0.63$ and $-0.40$ for the empirical-CVaR and EVT-based policies, respectively. The upper parts of the distributions are also comparable: the $95$th percentiles are $6.23$ and $6.11$, while the maximum observed errors are $50.09$ and $48.80$, respectively. Thus, under the general out-of-sample distribution, there is little evidence of a material difference in hedging performance between the two policies.

A markedly different behavior emerges when the evaluation is concentrated on extreme scenarios. In Panel B, the distribution of hedging errors for the EVT-trained policy shifts toward smaller values relative to the empirical-CVaR policy. The mean hedging error decreases from $17.49$ to $10.03$, while the median decreases from $16.52$ to $9.13$. Importantly, this difference remains visible further into the upper tail: the $70$th percentile decreases from $27.12$ to $17.50$, the $95$th percentile from $49.22$ to $35.46$, and the maximum observed hedging error from $114.27$ to $84.36$. The separation is therefore not driven exclusively by a small number of maximum observations; it is visible across several regions of the conditional hedging-loss distribution under the extreme-scenario evaluation.

\begin{figure}[H]\centering
    \caption{Distribution of out-of-sample hedging errors under the empirical and EVT-based optimization approaches.}
\begin{threeparttable}
    \includegraphics[width=16cm]{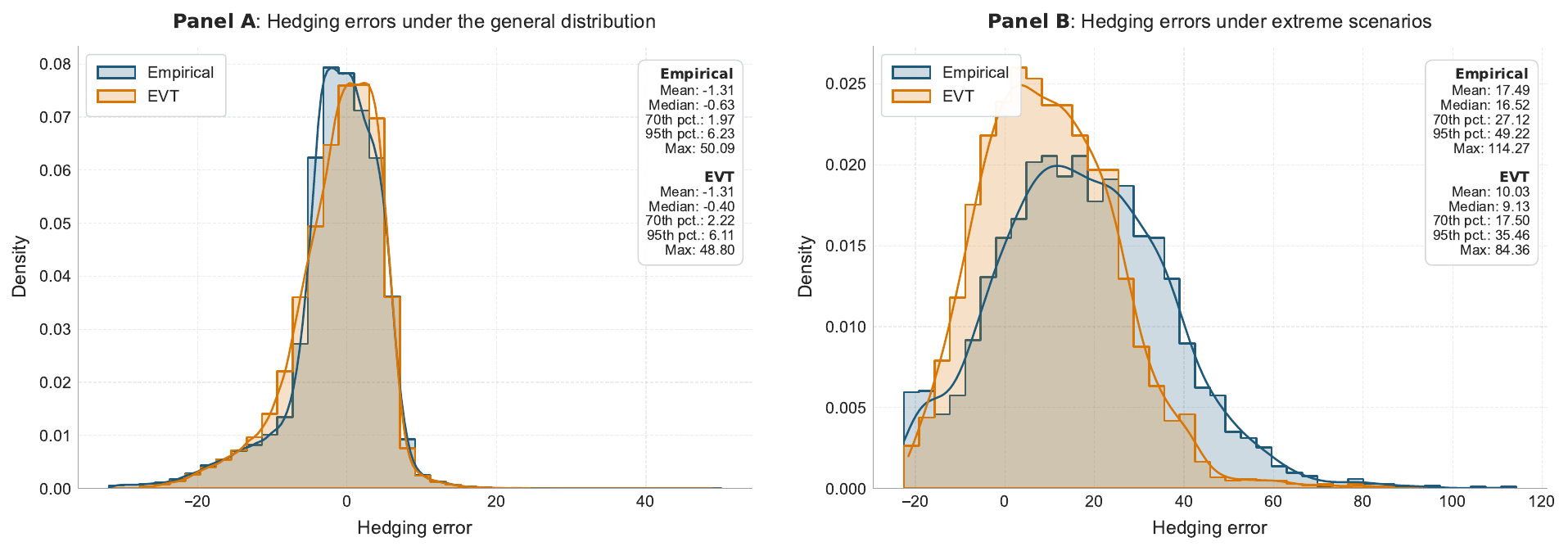}
    \begin{tablenotes}
    \item \footnotesize Hedging errors are evaluated on 100,000 and 3,000 out-of-sample paths under the general distribution and the conditional extreme-scenario evaluation, respectively. Panel A considers paths generated from the original market distribution, while Panel B considers trajectories from the same market model conditional on the event $\mathcal{E}_{0.999}$. Summary statistics report the mean, median, 70th and 95th percentiles, and maximum observed hedging error.
    \end{tablenotes}
\end{threeparttable}
    \label{fig:extreme_loss}
\end{figure}

These results suggest that the difference between the two optimization approaches becomes more pronounced under extreme market realizations. Under the general out-of-sample distribution, the two policies produce nearly indistinguishable mean errors and similar upper-tail statistics. Conditional on extreme market realizations, however, the EVT-trained policy produces smaller hedging errors throughout a substantial portion of the distribution, including its upper tail. This evidence is consistent with the hypothesis that explicitly modeling the tail during
optimization becomes consequential when the learned policy is exposed to scenarios in which extreme losses are more prevalent. Since the second evaluation is explicitly conditional on $\mathcal{E}_{0.999}$, however, these results should not be interpreted as evidence of an equivalent reduction in unconditional risk under the original market distribution.

\subsection{Sensitivity to Training Sample Size}
\label{se:sample_size}

The previous results show that the difference between empirical-CVaR and EVT-based optimization becomes more pronounced when the learned policies are exposed to extreme scenarios. This raises a related question: whether this behavior depends on the amount of data available during training. This is particularly relevant at the catastrophic confidence level $\alpha=0.999$, where only a small fraction of the simulated losses directly determines the empirical tail-risk objective. We therefore examine whether reducing the number of training paths materially changes the out-of-sample behavior of the resulting policies.

We train the empirical-CVaR and EVT-based policies using $50{,}000$, $100{,}000$, $200{,}000$, and $400{,}000$ simulated paths while maintaining the remaining training specifications unchanged. Each trained policy is subsequently evaluated on 100,000 out-of-sample
paths from the general distribution and on the 3,000-path conditional extreme-scenario sample defined in \autoref{se:extreme_loss}. \autoref{fig:sample_size} reports the resulting distributions of out-of-sample hedging errors.

Under the general distribution, shown in Panel A, the hedging-error distributions remain relatively stable as the number of training paths changes. The empirical-CVaR and EVT-based policies also exhibit broadly comparable out-of-sample behavior at each training-sample size. In particular, reducing the training sample from $400{,}000$ to $50{,}000$ paths does not produce a clear deterioration in the central part of either distribution. Thus, within the range of sample sizes considered here, the general out-of-sample performance appears relatively insensitive to the amount of training data.

The heavy-tail evaluation in Panel B provides a different perspective. Across all four training-sample sizes, the EVT-trained policies exhibit hedging-error distributions shifted toward smaller values relative to their empirical-CVaR counterparts. Importantly, this separation remains visible when the number of training paths is reduced. At the same time, the distributions do not exhibit a clear monotonic improvement as the training sample increases from $50{,}000$ to $400{,}000$. The experiment therefore does not support the stronger claim that increasing the number of training trajectories systematically improves either policy over this range. Rather, it suggests that the difference observed under the conditional extreme-scenario evaluation persists across the training-sample sizes considered.

\begin{figure}[H]\centering
    \caption{Out-of-sample hedging errors across training-sample sizes under the empirical and EVT-based optimization approaches.}
\begin{threeparttable}
    \includegraphics[width=16cm]{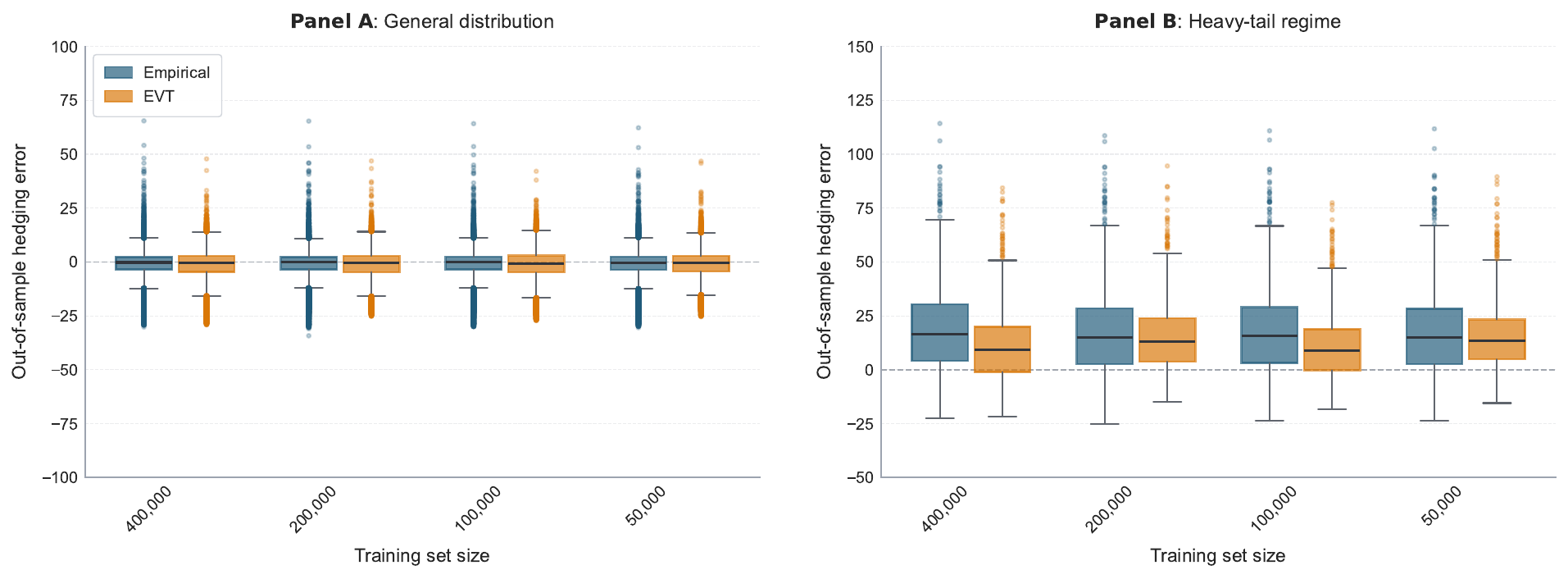}
    \begin{tablenotes}
    \item \footnotesize Policies are trained using $50{,}000$, $100{,}000$, $200{,}000$, and $400{,}000$ simulated paths. Panel A evaluates each policy on 100,000 independent paths from the general distribution, while Panel B evaluates each policy on the 3,000-path conditional extreme-scenario. Boxplots report the resulting out-of-sample hedging-error distributions for the empirical-CVaR and EVT-based policies.
    \end{tablenotes}
\end{threeparttable}
    \label{fig:sample_size}
\end{figure}

These results indicate that the relative behavior of the two optimization approaches is not driven exclusively by the largest training sample. Under the general distribution, changing the number of training paths has a limited visible effect on out-of-sample hedging errors, while under the heavy-tail regime the EVT-based policies exhibit distributions concentrated at smaller hedging errors at each of the training-sample sizes considered. This provides evidence that the extreme-loss behavior documented in Section~\ref{se:extreme_loss} persists when substantially fewer training trajectories are available.


\section{Conclusion}
\label{se:conclusion}

This study extends the EVT-based risk-aware reinforcement learning framework considered in the baseline methodology to policies parameterized by neural networks. The main limitation of the original approach in this setting arises from its reliance on finite-difference gradients, whose computational cost increases with the number of policy parameters. We address this limitation by expressing the POT-CVaR estimator as a differentiable empirical objective and computing its gradient jointly with respect to the neural-network parameters through automatic differentiation. The resulting framework preserves the EVT treatment of extreme losses while making it applicable to substantially richer policy parameterizations.

The theoretical results establish the properties supporting this construction. In particular, for a fixed threshold and under the stated local regularity conditions, the finite-sample POT-MOM objective is locally differentiable and automatic differentiation recovers its exact gradient. Moreover, as the number of threshold exceedances increases, the POT-MOM estimator converges in probability to its corresponding population POT-CVaR functional under the assumptions established in Section~\ref{se:theory}. These results characterize the differentiability and statistical behavior of the EVT component used within policy optimization, while keeping distinct the approximation of the true CVaR and the optimization of the neural-network parameters.

The numerical experiments complement these results from two perspectives. First, the controlled experiment provides an analytically tractable benchmark and shows that the proposed gradient-based procedure approaches the known optimal policy parameters. Second, the deep-hedging experiment illustrates the behavior of the methodology in a substantially more complex neural-network optimization problem. Under the general out-of-sample market distribution, the EVT and empirical-CVaR approaches ultimately provide comparable hedging performance, although the EVT-based objective exhibits a faster initial reduction and lower between-run variability in validation CVaR during training. The distinction becomes more pronounced under the conditional extreme-scenario evaluation: the EVT-trained policies exhibit substantially smaller hedging losses, while performance under the general market distribution remains broadly comparable.

Together, these results suggest that explicitly modeling the loss tail can be particularly useful when neural-network policies are trained to manage rare but severe outcomes. The benefit is therefore not that EVT uniformly improves policy
performance, but rather that it uses the exceedance distribution to model the tail in regimes where empirical CVaR relies on relatively few observations beyond the target quantile. Extending the analysis to a broader assessment of alternative tail models and risk measures would help determine how generally these benefits persist.


\vspace{-10pt}
\titleformat{\section}{\normalfont\large\bfseries} {\thesection}{1em}{#1}
\bibliographystyle{apalike}
\bibliography{references}  

@article{François2025,
  title={Enhancing Deep Hedging of Options with Implied Volatility Surface Feedback Information},
  author={François, Pascal and Gauthier, Geneviève and Godin, Fr{\'e}d{\'e}ric and P{\'e}rez Mendoza, Carlos Octavio},
  journal={Working paper},
  volume={},
  number={},
  pages={},
  year={2024},
  publisher={}
}

@article{François2025_stat_arbitrage,
  title={Is the Difference between Deep Hedging and Delta Hedging a Statistical Arbitrage?},
  author={François, Pascal and Gauthier, Geneviève and Godin, Fr{\'e}d{\'e}ric and P{\'e}rez Mendoza, Carlos Octavio},
  journal={Working paper},
  volume={},
  number={},
  pages={},
  year={2024},
  publisher={}
}

@article{davar2025catastrophic,
  title={Catastrophic-Risk-Aware Reinforcement Learning with Extreme-Value-Theory-Based Policy Gradients},
  author={Davar, Parisa and Godin, Fr{\'e}d{\'e}ric and Garrido, Jos{\'e}},
  journal={The Journal of Finance and Data Science},
  volume={11},
  pages={100165},
  year={2025},
  doi={10.1016/j.jfds.2025.100165}
}

@article{bader2018automated,
  title={Automated threshold selection for extreme value analysis via ordered goodness-of-fit tests with adjustment for false discovery rate},
  author={Bader, Brian and Yan, Jun and Zhang, Xuebin},
  journal={The Annals of Applied Statistics},
  year={2018},
  volume={12},
  number={1},
  pages={310--329}
}

@article{balkema1974residual,
  title={Residual life time at great age},
  author={Balkema, August A and De Haan, Laurens},
  journal={The Annals of Probability},
  volume={2},
  number={5},
  pages={792--804},
  year={1974},
  publisher={Institute of Mathematical Statistics}
}

@article{buehler2019deep,
  title={Deep hedging},
  author={Buehler, Hans and Gonon, Lukas and Teichmann, Josef and Wood, Ben},
  journal={Quantitative Finance},
  volume={19},
  number={8},
  pages={1271--1291},
  year={2019},
  publisher={Taylor \& Francis}
}

@book{coles2001introduction,
  title={An introduction to statistical modeling of extreme values},
  author={Coles, Stuart and Bawa, Joanna and Trenner, Lesley and Dorazio, Pat},
  volume={208},
  year={2001},
  publisher={Springer}
}

@book{mcneil2015quantitative,
  title={Quantitative risk management: concepts, techniques and tools-revised edition},
  author={McNeil, Alexander J and Frey, R{\"u}diger and Embrechts, Paul},
  year={2015},
  publisher={Princeton university press}
}

@article{pickands1975statistical,
  title={Statistical inference using extreme order statistics},
  author={Pickands III, James},
  journal={The Annals of Statistics},
  pages={119--131},
  year={1975},
  publisher={JSTOR}
}

@inproceedings{troop2021bias,
  title={Bias-corrected peaks-over-threshold estimation of the cvar},
  author={Troop, Dylan and Godin, Fr{\'e}d{\'e}ric and Yu, Jia Yuan},
  booktitle={Uncertainty in Artificial Intelligence},
  pages={1809--1818},
  year={2021},
  organization={PMLR}
}

@article{rockafellar2002conditional,
  title={Conditional value-at-risk for general loss distributions},
  author={Rockafellar, R Tyrrell and Uryasev, Stanislav},
  journal={Journal of Banking \& Finance},
  volume={26},
  number={7},
  pages={1443--1471},
  year={2002},
  publisher={Elsevier}
}

\newpage

\appendix

\appendix

\appendix

\section{Details for the EVT-based training approach}
\label{appen:EVTtraining}

The EVT-based training approach consists of three main steps. First, the threshold above which the loss distribution is approximated by a GPD is selected using a sequence of Anderson--Darling goodness-of-fit tests combined with the ForwardStop procedure. Second, conditional on the selected threshold, the CVaR training objective is estimated using the POT-MOM estimator when sufficient exceedances are available, with empirical CVaR used as a fallback otherwise. Third, the resulting objective is differentiated with respect to the neural-network parameters using automatic differentiation. Threshold selection is performed outside the differentiation step, and the selected threshold is treated as fixed while the active finite-sample objective is differentiated.

\subsection{Threshold selection}
\label{appen:threshold_selection}

Let $\{X_i\}_{i=1}^{n}$ denote the sample of losses generated under the current policy. We consider a grid of candidate percentiles $\mathcal{P}=\{p_1,\ldots,p_K\}$. For each candidate percentile, the corresponding threshold is obtained from the empirical distribution, a GPD is fitted to the resulting excesses using the method of moments, and an Anderson--Darling goodness-of-fit statistic and its corresponding $p$-value are computed. Candidate thresholds associated with a shape parameter above the prescribed cutoff are discarded. The ForwardStop procedure is then
applied to the sequence of retained $p$-values to determine the threshold used in the CVaR estimation. For the numerical experiments, we consider a grid $\mathcal{P}$ of
$20$ equally spaced candidate percentiles between $0.80$ and $0.99$, with ForwardStop significance level $\gamma=0.1$ and shape-parameter cutoff $\xi_{\max}=0.99$.

\begin{algorithm}[H]
\caption{Threshold selection procedure.}
\label{algo:EVT_threshold}
\begin{algorithmic}
\Require Loss sample $\{X_i\}_{i=1}^{n}$, percentile grid
$\mathcal{P}=\{p_1,\ldots,p_K\}$, significance level $\gamma$, and
shape-parameter cutoff $\xi_{\max}$.

\State Initialize the lists of admissible candidates $\mathbb{A}$ and
$p$-values $\mathbb{P}$

\For{$j=1$ to $K$}
    \Comment{Evaluate candidate threshold}
    \State Set $u_j$ equal to the empirical $p_j$-quantile of
    $\{X_i\}_{i=1}^{n}$
    \State Compute the excesses
    $\mathcal{Y}_j=\{X_i-u_j:X_i>u_j\}$
    \State Estimate $(\widehat{\xi}_j,\widehat{\sigma}_j)$ from
    $\mathcal{Y}_j$ using the method of moments
    \If{$\widehat{\xi}_j\leq\xi_{\max}$}
        \State Compute the Anderson--Darling statistic $A_j^2$ for the
        fitted GPD
        \State Compute its corresponding $p$-value $p_j^{\mathrm{AD}}$
        \State Append
        $(u_j,\widehat{\xi}_j,\widehat{\sigma}_j,p_j)$ to $\mathbb{A}$
        and $p_j^{\mathrm{AD}}$ to $\mathbb{P}$
    \EndIf
\EndFor

\If{$\mathbb{A}=\emptyset$}
    \State No admissible EVT threshold is identified
\Else
    \For{$k=1$ to $\#\mathbb{P}$}
        \State Calculate
        $W_k=\frac{1}{k}\sum_{j=1}^{k}-\log(1-p_j^{\mathrm{AD}})$
    \EndFor
    \State Determine the ForwardStop index from
    $\{k:W_k\leq\gamma\}$
    \State Select the corresponding candidate threshold $u$
\EndIf

\State \Return selected threshold $u$, if available
\end{algorithmic}
\end{algorithm}

\subsection{Hybrid EVT--empirical CVaR estimation}
\label{appen:hybrid_cvar}

Once a threshold $u$ has been selected, it is treated as fixed in the differentiable CVaR computation. The excesses above $u$ are used to estimate the GPD parameters through the method of moments. When the sample contains more than one exceedance, the POT-MOM estimator is used. Otherwise, the
procedure falls back to the empirical CVaR computed directly from the upper $\alpha$-tail of the loss sample.

\begin{algorithm}[H]
\caption{Hybrid EVT--empirical CVaR estimation.}
\label{algo:hybrid_cvar}
\begin{algorithmic}
\Require Loss sample $\{X_i\}_{i=1}^{n}$, CVaR level $\alpha$, $\varepsilon = 10^{-8}$, and
selected threshold $u$.

\State Compute the excesses
$\mathcal{Y}=\{X_i-u:X_i>u\}$ and let $n_u=\#\mathcal{Y}$

\If{$n_u>1$}
    \Comment{EVT branch}
    \State Compute the sample mean $\bar{Y}$ and sample variance $S_Y^2$
    of the excesses
    \State Calculate the MOM shape estimate
    $\widehat{\xi}_{\text{MOM}}=\frac{1}{2}\left(1-\frac{\bar{Y}^{2}}{S_Y^{2}}\right)$
    \State Set $\widehat{\xi}
=
\min\left\{
\max\left\{\widehat{\xi}_{\mathrm{MOM}},0\right\},
\frac{1}{2}-\varepsilon
\right\}.$
    \State Calculate
    $\widehat{\sigma}=\bar{Y}(1-\widehat{\xi})$
    \State Estimate the exceedance probability
    $\widehat{p}_u=\frac{1}{n}\sum_{i=1}^{n}\mathds{1}_{\{X_i>u\}}$
    \State Set $s=\widehat{p}_u/(1-\alpha)$
    \If{$|\widehat{\xi}|<10^{-6}$}
        \State Calculate
        $\widehat{\mathrm{CVaR}}_{\alpha}
        =u+\widehat{\sigma}\left(\log(s)+1\right)$
    \Else
        \State Calculate
        $\widehat{\mathrm{CVaR}}_{\alpha}
        =u+\frac{\widehat{\sigma}}{1-\widehat{\xi}}
        \left(1+\frac{s^{\widehat{\xi}}-1}{\widehat{\xi}}\right)$
    \EndIf
\Else
    \Comment{Empirical fallback}
    \State Sort the losses
    $X_{(1)}\leq\cdots\leq X_{(n)}$
    \State Set $j_{\alpha}=\lfloor\alpha n\rfloor$
    \State Calculate
    $\widehat{\mathrm{CVaR}}_{\alpha}
    =\frac{1}{n-j_{\alpha}}\sum_{j=j_{\alpha}+1}^{n}X_{(j)}$
\EndIf

\State \Return $\widehat{\mathrm{CVaR}}_{\alpha}$
\end{algorithmic}
\end{algorithm}

\subsection{Training approach for the neural-network policy}
\label{appen:EVTtraining_strategy}

We now combine the threshold-selection and CVaR-estimation procedures within the neural-network optimization. At training iteration $m$, the complete sample of losses is generated under the current policy $\pi_{\theta_m}$. The threshold-selection procedure in Algorithm~\ref{algo:EVT_threshold} is then applied to this sample. Conditional on the selected threshold, the hybrid CVaR estimator in Algorithm~\ref{algo:hybrid_cvar} defines the objective used to update the policy. Importantly, the selected threshold is held fixed during differentiation. Consequently, gradients propagate through the simulated losses, excesses, empirical moments, GPD parameter estimates, and CVaR estimator, but not through the threshold-selection procedure.

\begin{algorithm}[H]
\caption{EVT-based neural-network training approach.}
\label{algo:EVT_training}
\begin{algorithmic}
\Require Initial policy parameters $\theta_1$, CVaR level $\alpha$,
learning-rate sequence $\{\eta_m\}_{m=1}^{M}$, number of training iterations
$M$, and sample size $n$.

\For{$m=1$ to $M$}

    \Comment{Generate the complete loss sample}
    \State Generate $n$ trajectories under the current policy
    $\pi_{\theta_m}$
    \State Compute the corresponding loss sample
    $\{X_{\theta_m,i}\}_{i=1}^{n}$

    \Comment{Select threshold}
    \State Apply Algorithm~\ref{algo:EVT_threshold} to
    $\{X_{\theta_m,i}\}_{i=1}^{n}$
    \State Obtain the selected threshold $u_m$, when available

    \Comment{Construct the CVaR training objective}
    \If{an admissible threshold $u_m$ is identified}
        \State Hold $u_m$ fixed during differentiation
        \State Compute
        $\widehat{J}^{\mathrm{H}}_{n,\alpha}(\theta_m)$ using
        Algorithm~\ref{algo:hybrid_cvar}
    \Else
        \State Compute
        $\widehat{J}^{\mathrm{H}}_{n,\alpha}(\theta_m)$ using the
        empirical CVaR fallback
    \EndIf

    \Comment{Update policy parameters}
    \State Compute
    \[
        g^{\mathrm{H}}_{n,\alpha}(\theta_m)
        =
        \left.
        \nabla_{\theta}
        \widehat{J}^{\mathrm{H}}_{n,\alpha}(\theta)
        \right|_{\theta=\theta_m}
    \]
    using automatic differentiation
    \State Update $\theta_{m+1}$ using $g^H_{n,\alpha}(\theta_m)$ and the selected gradient-based optimizer. 

\EndFor
\end{algorithmic}

\smallskip\noindent\textbf{Note:}
At each training iteration, threshold selection and CVaR estimation are performed using the complete loss sample rather than separate mini-batches. The threshold-selection procedure is not differentiated. When the EVT branch is active, the selected threshold is held fixed while automatic differentiation preserves the dependence of the loss sample, excesses, empirical moments, MOM parameter estimates, and resulting POT-CVaR estimator on the neural-network parameters. When the EVT approximation cannot be used, the objective falls back to empirical CVaR.
\end{algorithm}


\end{document}